\documentclass[journal]{IEEEtran}

\usepackage{amsmath,amssymb,amsthm}
\usepackage{booktabs}
\usepackage{graphicx}
\usepackage{tikz}
\usetikzlibrary{arrows.meta,calc,fit,positioning,shapes.geometric}
\usepackage{cite}
\usepackage{url}
\usepackage{hyperref,doi}
\hypersetup{
    colorlinks=true,
    linkcolor=blue,
    citecolor=blue,
    urlcolor=blue
}
\usepackage{algorithm}
\usepackage{algpseudocode}
\newtheorem{proposition}{Proposition}
\newtheorem{corollary}{Corollary}
\newcommand{\R}{\mathbb{R}}
\newcommand{\Wrec}{W_{\mathrm{rec}}}
\newcommand{\Win}{W_{\mathrm{in}}}

\begin{document}

\title{Lindblad-Inspired Multi-Timescale Reservoir Computing with Separable Rotation and Dissipation}

\author{Jyotiranjan~Beuria,
        Amit~Shukla%
\thanks{J. Beuria is with the IKS Research Centre, ISS, Delhi, India
(e-mail: jyotiranjan.beuria@gmail.com).}%
\thanks{A. Shukla is with the Indian Institute of Technology Mandi,
Himachal Pradesh, India (e-mail: amit.shukla@iitmandi.ac.in).}}

\maketitle

\begin{abstract}
Echo-state networks enable efficient temporal learning by fixing the recurrent dynamics and training only a linear readout. However, conventional reservoirs typically accommodate signal mixing, memory retention, and stability within a single random recurrent matrix. Existing structured designs improve topology, norm preservation, leakage, or depth, but generally do not provide separate modal control of reversible mixing and irreversible forgetting together with a direct global stability guarantee. We introduce a classical Lindblad-inspired multi-timescale reservoir that bridges open-system dynamical principles with structured state-space modeling. The recurrent operator is assembled from exactly discretized damped rotational modes, so rotation and decay become independent design variables governing phase mixing and memory loss. Orthogonal mode mixing preserves normality, while the decay spectrum directly determines the echo-state stability margin without post-hoc spectral-radius rescaling. We evaluate the method over ten aligned seeds against standard, leaky, deep, orthogonal, cycle, and next-generation reservoirs, together with a compact trained gated recurrent unit, across linear memory, nonlinear recurrence, chaotic forecasting, delayed logic, and real sensor calibration. 
Across the benchmark suite, the proposed reservoir achieves the best fixed-reservoir performance on bounded NARMA-20 and the lowest mean error on Lorenz-63, matches the strongest linear-memory result, and remains broadly competitive across NARMA-10, Mackey-Glass forecasting, delayed XOR, and real-world air-quality calibration.
Ablation studies show that rotation increases state diversity, whereas dissipation provides controlled forgetting and improves predictive conditioning. The resulting framework offers an interpretable recurrent architecture in which mixing, memory, and stability are explicit and independently tunable design variables.
\end{abstract}

\begin{IEEEkeywords}
Reservoir computing, echo-state network, structured recurrent network,
Lindblad dynamics, fading memory, echo-state property, temporal learning.
\end{IEEEkeywords}

\section{Introduction}
\label{sec:introduction}

\IEEEPARstart{R}{eservoir} computing fixes a nonlinear recurrent system and
trains only a linear readout~\cite{jaeger2001,maass2002}. It consequently
avoids backpropagation through time (BPTT), supports closed-form training,
and maps naturally to physical substrates~\cite{tanaka2019,nakajima2020}.
For the resulting state to be a well-defined function of the input history,
the recurrence must forget its initial condition: the echo-state property
(ESP)~\cite{jaeger2001,manjunath2013}. Once this fading-memory requirement
and sufficient state richness are available, fixed reservoirs with linear
readouts form universal approximators of fading-memory
filters~\cite{boyd1985,grigoryeva2018}.

The standard echo-state network (ESN) \cite{sun2022systematic} constructs a dense random recurrent
matrix and rescales its spectral radius. This recipe is effective but
architecturally opaque. The same matrix is expected to mix coordinates,
retain useful transients, and contract initial-state perturbations, and its
entries have no direct operational interpretation. Structured reservoirs
replace this matrix with cycles \cite{rodan2011,rodan2012}, orthogonal transformations \cite{strauss2012design,saxe2014}, or layered
compositions~\cite{gallicchio2017}. These
models establish that random connectivity is not essential, but they do
not explicitly parameterize reversible mixing and irreversible forgetting
as separate mechanisms.

Open-system dynamics provide a useful decomposition. In a Bloch-vector
representation, the homogeneous part of a finite-dimensional
Gorini-Kossakowski-Sudarshan-Lindblad (GKSL) evolution contains a
skew-symmetric rotational contribution and a dissipative
contribution~\cite{gks1976,lindblad1976,petruccione2002}. Quantum reservoir
computing uses such dynamics in a quantum substrate and obtains features
from measured observables~\cite{fujii2017,mujal2021,sannia2024}. Here we
ask a different question: can the rotation-dissipation split serve as a
classical recurrent-matrix design principle?

We address this question through a classical reservoir architecture motivated by the structural decomposition used in open-system dynamics. Each recurrent component is built from an exactly solvable mode that combines rotation with exponential decay, and six independently scaled components are concatenated to span multiple memory timescales. An orthogonal transformation then mixes the modal coordinates while preserving their individual decay properties. The term \emph{Lindblad-inspired} is therefore used in a precise and limited sense: the model adopts the generator-level separation between reversible rotation and irreversible dissipation, but it does not necessarily require a density-matrix state space, complete positivity, trace preservation, or quantum measurement. The proposed system is consequently a classical structured reservoir rather than a simulation or physical realization of a quantum channel. This study makes four contributions listed below.

\begin{enumerate}
\item The recurrent operator is obtained in closed form with separate
parameters for phase rotation and radial decay in a multi-timescale scenario. It is not subsequently
rescaled or spectrally normalized; consequently, the prescribed decay rates
retain their direct interpretation as memory and contraction timescales.

\item The resulting operator is normal, so its singular values and contraction
factor can be computed exactly. In particular, its induced two-norm is bounded
by the exponential of the smallest decay rate and is therefore strictly below
unity. This provides direct echo-state stability without
substituting spectral-radius control for singular-value control in a
non-normal matrix or making assumptions about products of non-normal
Jacobians.

\item A matched-feature evaluation compares the competitiveness of the proposed architecture with a
standard ESN \cite{jaeger2001}, leaky ESN \cite{jaeger2007leaky}, two-layer Deep ESN \cite{gallicchio2017}, orthogonal reservoir, cycle
reservoir with jumps (CRJ) \cite{rodan2012}, next-generation reservoir computing
(NG-RC)~\cite{gauthier2021ngrc}, and a compact end-to-end-trained
GRU~\cite{cho2014gru}.

\item Mechanism ablations, statistical comparisons, and training-time
measurements identify the regimes in which the proposed design is most
effective. Improvements on linear memory and bounded NARMA-20 are reported
alongside cases where task-specific leaky, polynomial-delay, or end-to-end
trained recurrent models perform better.
\end{enumerate}

The remainder of the paper first reviews the closest structured-reservoir
and open-system approaches, then introduces the architecture and establishes
its stability guarantee. The subsequent sections present the benchmark
results, mechanism ablations, computational analysis, limitations, and
broader implications.

\section{Related Work}
\label{sec:related}

\subsection{Structured and Contemporary Reservoirs}

Cycle reservoirs and CRJ replace random connectivity with a unidirectional
ring \cite{rodan2011} and regularly spaced jumps~\cite{rodan2012}. 
Orthogonal echo-state networks instead use fixed
norm-preserving recurrent connectivity, providing a rotation-dominated
reservoir reference \cite{mayer2017orthogonal}. Related orthogonal and
unitary constraints have also been studied in end-to-end-trained neural
networks to improve long-range signal and gradient propagation
\cite{saxe2014,arjovsky2016}. Deep echo-state networks construct hierarchical
temporal representations by stacking multiple fixed recurrent layers
\cite{gallicchio2017}.

These structures expose a useful distinction. A cycle specifies
\emph{where} information moves, an orthogonal map specifies how Euclidean
norm is preserved, and a deep reservoir specifies how representations are
composed. None of these choices, by itself, attaches separate parameters to
phase rotation and amplitude decay. Leaky integration does expose a decay
control, but it blends the previous state with a newly activated state and
therefore changes both the effective time constant and nonlinear operating
point. The architecture studied here instead assigns an angle and a radial
decay to each linear mode before applying the common activation.

NG-RC takes a complementary route. It replaces an explicit recurrent
network by a nonlinear vector autoregression over delayed inputs, followed
by a linear readout~\cite{gauthier2021ngrc}. Its explicit polynomial
features can be exceptionally efficient when the required interaction
appears inside the chosen delay window. On the other hand, a GRU is not a reservoir because its recurrent weights are optimized
end-to-end~\cite{cho2014gru}. It is nevertheless an informative reference:
it shows what task-specific recurrent training can gain and what it costs.
We therefore compare predictive performance, trainable parameter count, and
wall-clock training time, while treating the compact GRU as a representative
end-to-end-trained recurrent baseline.

The comparison presented in this work is deliberately broader than a standard ``ESN versus new proposed
ESN'' study. Deep ESN tests whether a hierarchy of fixed recurrent
representations is sufficient; NG-RC tests whether an explicit finite delay
embedding is preferable to recurrence; and the GRU tests the benefit of
optimizing the recurrent state transition. These controls help identify
whether a result follows from the proposed mechanism or merely from using a
temporal model.

\subsection{Echo-State-Property Analysis}

A recurrent spectral radius below unity is widely used as a practical design
rule, but it does not generally guarantee contraction when the recurrent
operator is non-normal. A stronger sufficient condition controls the largest
singular value, thereby ensuring that differences between reservoir states
decrease uniformly for every admissible input sequence
\cite{buehner2006,yildiz2012}. This condition is sufficient rather than
necessary: input-driven reservoirs may still forget their initial conditions
even when no uniform global contraction bound is available
\cite{manjunath2013}. Recent reviews continue to emphasize this distinction
between easily checked matrix criteria and the input-dependent nature of the
echo-state property \cite{sun2022systematic}. In this work, we use the operator norm to establish input-uniform echo-state stability and the Jacobian-product exponent only to diagnose driven stability.

\subsection{Open-System and Quantum Reservoir Computing}

The Gorini-Kossakowski-Sudarshan-Lindblad (GKSL) master equation characterizes
finite-dimensional Markovian quantum dynamical semigroups
\cite{gks1976,lindblad1976}. Quantum reservoir computing instead uses the
evolution of a genuine quantum state, together with measured observables, to
construct temporal features for learning tasks
\cite{fujii2017,suzuki2022,kobayashi2024_esp}. Recent work further shows that
controlled dissipation can enhance, rather than merely degrade, computational
performance in quantum reservoirs \cite{sannia2024}. Our model transfers only
the structural separation between reversible rotation and irreversible decay
to a classical state space. This distinction is essential: a generic real
damped rotation need not represent a completely positive, trace-preserving
quantum channel in Bloch coordinates.

\subsection{Design Gap and Evaluation Questions}

The literature leaves a specific design gap between random ESNs and
physical quantum reservoirs. It is possible to borrow a mechanism-level
organization from open-system dynamics without claiming a quantum
implementation. Doing so is useful only if four questions are answered.
First, does the parameterization provide a stability statement stronger
than spectral-radius rescaling? Second, do rotation and decay make
empirically distinguishable contributions? Third, is the resulting
reservoir competitive against recent structured alternatives at a matched
feature budget? Fourth, what is gained or lost relative to training a small
recurrent neural network end to end? The theory, ablation, contemporary
baselines, and cost audit presented in this work are organized around these four questions.

\section{Lindblad-Inspired Reservoir}
\label{sec:method}

\begin{figure*}[t]
\centering
\begin{tikzpicture}[
    x=1cm,y=1cm,
    every node/.style={font=\small},
    title/.style={font=\small\bfseries,align=center},
    block/.style={
        draw=black!65,rounded corners=1.5pt,minimum height=0.72cm,
        align=center,inner xsep=4pt,inner ysep=2pt,line width=0.55pt},
    mode/.style={block,fill=blue!7,draw=blue!65!black},
    flow/.style={block,fill=orange!8,draw=orange!75!black},
    fixed/.style={block,fill=black!3},
    state/.style={block,fill=blue!6,draw=blue!60!black},
    fitted/.style={block,fill=green!8,draw=green!55!black,line width=0.7pt},
    member/.style={
        block,minimum width=1.55cm,minimum height=0.40cm,
        inner ysep=1pt,font=\footnotesize},
    note/.style={font=\footnotesize,align=center},
    sum/.style={
        circle,draw=black!70,fill=white,minimum size=0.54cm,
        inner sep=0pt,line width=0.6pt,font=\normalsize},
    arrow/.style={-{Latex[length=2.0mm,width=1.35mm]},line width=0.62pt},
    guide/.style={arrow,draw=black!70},
    frozen/.style={guide,draw=blue!55!black},
    train/.style={guide,draw=green!45!black}
]

% (a) Exact construction of one member.
\node[title] at (2.15,5.15) {(a) Fixed member construction};
\node[mode,minimum width=3.45cm] (gen) at (2.15,4.35)
    {$A_{mj}=-\gamma_{mj}I+\omega_{mj}J$};
\node[flow,minimum width=3.45cm] (pole) at (2.15,3.18)
    {$B_{mj}=e^{-\gamma_{mj}}R(\omega_{mj})$};
\node[state,minimum width=4.05cm] (wm) at (2.15,1.73)
    {$W_m=Q_m\operatorname{blkdiag}(B_{mj})Q_m^\top$};
\draw[arrow] (gen) -- node[right,note] {exact flow $e^{A_{mj}}$} (pole);
\draw[arrow] (pole) --
    node[right,note,align=left]
    {$|\lambda_{mj}|=e^{-\gamma_{mj}}$\\
     $\arg\lambda_{mj}=\pm\omega_{mj}$} (wm);
\node[note] at (2.15,0.82)
    {$\|W_m\|_2=e^{-\min_j\gamma_{mj}}<1$};
\node[note] at (2.15,0.43)
    {25 mode pairs; no spectral rescaling};

% Panel divider.
\draw[black!18] (4.55,0.35) -- (4.55,5.38);

% (b) Implemented update of one member.
\node[title] at (7.65,5.15) {(b) One-member update};
\node[flow,minimum width=0.85cm] (xt) at (5.15,4.13) {$x_t$};
\node[fixed,minimum width=1.18cm] (win) at (6.45,4.13)
    {$W_{\rm in}^{(m)}$};
\node[state,minimum width=0.96cm] (zt) at (5.15,2.48)
    {$z_t^{(m)}$};
\node[fixed,minimum width=1.18cm] (recur) at (6.45,2.48) {$W_m$};
\node[sum] (add) at (7.78,3.30) {$+$};
\node[note] (bias) at (7.78,4.13) {$b_m$};
\node[fixed,minimum width=0.88cm] (act) at (8.80,3.30) {$\tanh$};
\node[state,minimum width=1.03cm] (znext) at (10.08,3.30)
    {$z_{t+1}^{(m)}$};
\draw[frozen] (xt) -- (win);
\draw[frozen] (win) -- (add);
\draw[frozen] (zt) -- (recur);
\draw[frozen] (recur) -- (add);
\draw[frozen] (bias) -- (add);
\draw[frozen] (add) -- (act);
\draw[frozen] (act) -- (znext);
\coordinate (delayright) at (10.08,1.36);
\coordinate (delayleft) at (5.15,1.36);
\draw[frozen] (znext.south) -- (delayright) --
    node[below,note] {one-step delay} (delayleft) -- (zt.south);
\node[note] at (7.65,0.48)
    {$z_0^{(m)}=0$; $W_m,W_{\rm in}^{(m)},b_m$ are frozen};

% Panel divider.
\draw[black!18] (10.72,0.35) -- (10.72,5.38);

% (c) Six parallel members and the fitted path.
\node[title] at (14.43,5.15) {(c) Ensemble and readout};
\node[note] at (14.25,4.80)
    {six parallel members, $d_m=50$};
\node[member,minimum width=0.92cm,fill=blue!8] (m1) at (11.22,4.27)
    {$z_{t+1}^{(1)}$};
\node[member,minimum width=0.92cm,fill=green!8] (m2) at (12.43,4.27)
    {$z_{t+1}^{(2)}$};
\node[member,minimum width=0.92cm,fill=yellow!14] (m3) at (13.64,4.27)
    {$z_{t+1}^{(3)}$};
\node[member,minimum width=0.92cm,fill=orange!10] (m4) at (14.85,4.27)
    {$z_{t+1}^{(4)}$};
\node[member,minimum width=0.92cm,fill=violet!8] (m5) at (16.06,4.27)
    {$z_{t+1}^{(5)}$};
\node[member,minimum width=0.92cm,fill=blue!8] (m6) at (17.27,4.27)
    {$z_{t+1}^{(6)}$};
\node[fixed,minimum width=6.82cm,minimum height=0.52cm] (concat)
    at (14.25,3.37)
    {concatenate \quad $z_{t+1}
     =\operatorname{col}_{m=1}^{6}z_{t+1}^{(m)}\in\R^{300}$};
\node[fixed,minimum width=4.55cm,minimum height=0.58cm] (scale)
    at (14.25,2.48)
    {train-fit scaling \quad
     $\widetilde z=(z-\mu_z)\oslash\sigma_z$};
\node[fixed,minimum width=5.75cm,minimum height=0.58cm] (feature)
    at (14.25,1.57)
    {$\phi_t=\widetilde z_t$ (regression)
     \quad or \quad
     $\phi_t=[\widetilde z_t;\widetilde z_t^{\odot2}]$ (XOR)};
\node[fitted,minimum width=2.70cm,minimum height=0.70cm] (ridge)
    at (14.25,0.66)
    {\textbf{fitted ridge readout}\\
     $(W_{\rm out},c_{\rm out})$};
\node[note] (yhat) at (16.35,0.66) {$\widehat y_t$};

\foreach \src in {m1,m2,m3,m4,m5,m6}
    \draw[frozen] (\src.south) -- (\src.south |- concat.north);
\draw[guide] (concat) -- (scale);
\draw[guide] (scale) -- (feature);
\draw[train] (feature) -- (ridge);
\draw[train] (ridge) -- (yhat);

\end{tikzpicture}
\caption{Implemented architecture. (a) Each recurrent member is built from
25 exact damped-rotation flows, followed by an orthogonal similarity
transform; no spectral rescaling is applied. (b) At every time step, the
fixed recurrent state, fixed input projection, and fixed bias are summed
before the elementwise nonlinearity; the curved arrow is the one-step
recurrent delay. (c) Six 50-state members are concatenated, standardized
with fitted-sample statistics, passed through the task-specific feature map,
and mapped to the target. Blue/black arrows denote frozen computation;
green arrows denote the only fitted path. The squared feature branch is used
only for delayed XOR.}
\label{fig:architecture}
\end{figure*}
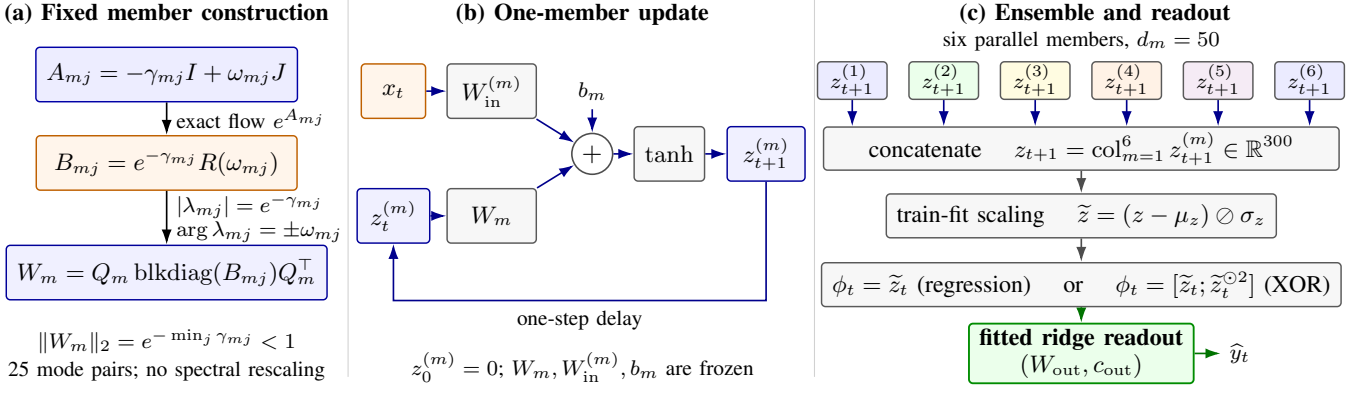

\subsection{From Generator Decomposition to Damped Modes}

To ground the architectural design, we first trace how open-system dynamics motivate the separation of rotation and dissipation, and how this continuous physical template translates step-by-step into an exact discrete-time classical reservoir (illustrated in Fig.~\ref{fig:architecture}).

\subsubsection{\textbf{Step 1: Open-System Motivation and Generator Decomposition}}
In quantum mechanics, the Markovian evolution of a $d$-level open system density matrix $\rho$ is governed by the Gorini-Kossakowski-Sudarshan-Lindblad (GKSL) master equation~\cite{gks1976,lindblad1976}:
\begin{equation}
\dot{\rho}=-i[H,\rho]+\sum_k \gamma_k
\left(L_k\rho L_k^\dagger-\tfrac12\{L_k^\dagger L_k,\rho\}\right),
\label{eq:gksl}
\end{equation}
where $H$ is the system Hamiltonian driving coherent unitary rotation, $L_k$ are Lindblad jump operators modeling environmental coupling, and $\gamma_k \ge 0$ are decay rates. 

To convert this operator dynamics into real vector coordinates, let $\{F_a\}_{a=1}^{d^2-1}$ be an orthonormal, traceless Hermitian operator basis satisfying $\operatorname{Tr}(F_aF_b)=2\delta_{ab}$. Expanding the density matrix as
\begin{equation}
\rho=\frac{I_d}{d}+\frac12\sum_{a=1}^{d^2-1}r_a F_a
\label{eq:bloch_expansion}
\end{equation}
yields a real-valued state vector $r = (r_1, \dots, r_{d^2-1})^\top \in \R^{d^2-1}$, known as the Bloch vector. Projecting the master equation \eqref{eq:gksl} onto this basis yields an affine linear differential equation for $r$~\cite{petruccione2002}:
\begin{equation}
\dot r = (\Omega + D)r + c, \qquad \Omega^\top = -\Omega,
\label{eq:bloch}
\end{equation}
where $\Omega \in \R^{(d^2-1)\times(d^2-1)}$ is a skew-symmetric matrix representing reversible coherent phase rotation, $D \in \R^{(d^2-1)\times(d^2-1)}$ is a dissipative matrix representing irreversible amplitude damping, and $c \in \R^{d^2-1}$ is an affine vector representing relaxation toward a stationary state.

\subsubsection{\textbf{Step 2: Classical Reservoir Parameterization and 2D Mode Pairs}}
In classical reservoir computing, the role of the affine drive $c$ (driving the state toward a fixed point) is already fulfilled by the external input drive $\Win x_t$ and fixed bias $b$. Therefore, our recurrent matrix design focuses exclusively on the linear homogeneous generator $\Omega + D$. We adopt the rotation-dissipation split in \eqref{eq:bloch} as an architectural template to parameterize reversible mixing and fading memory separately.

\paragraph*{Classical Algorithm vs. Quantum Substrate} We emphasize that the term \emph{Lindblad-inspired} denotes an architectural transfer of the generator-level rotation-dissipation split ($\Omega + D$) to a classical recurrent network design, not a physical quantum simulation. The algorithm operates entirely within a classical real vector space $\mathbb{R}^{d_z}$ under elementwise non-linearities ($\tanh$) and standard linear readouts. It does not simulate or implement a physical quantum channel, density matrix, or quantum hardware. In particular, density-matrix positivity, trace preservation ($\operatorname{Tr}(\rho)=1$), Hilbert-space tensor structures, and quantum measurement back-action play no role in this classical algorithm.

While a single physical 2-level qubit yields a 3D Bloch vector, a 2D or 3D state space is far too low-dimensional for complex machine learning sequence tasks. To build a high-capacity recurrent member with state dimension $d_m = 50$, we structure the continuous generator $\Omega + D$ as a direct sum of independent two-dimensional modal subspaces (Fig.~\ref{fig:architecture}(a)). Because each elementary damped rotational mode operates on a 2D coordinate plane ($2 \times 2$ matrix block), a 50-dimensional member is assembled from exactly $d_m / 2 = 25$ independent 2D mode pairs (spanning $25 \times 2 = 50$ state coordinates).

For a single 2D mode pair $j \in \{1, \dots, 25\}$ of member $m \in \{1, \dots, M\}$:
\begin{enumerate}
\item \textbf{Rotational Generator ($\Omega_{mj}$):} In 2D, the fundamental skew-symmetric generator of rotation is $J = \begin{pmatrix}0 & -1 \\ 1 & 0\end{pmatrix}$, scaled by an angular frequency parameter $\omega_{mj} \in \R$.
\item \textbf{Dissipative Generator ($D_{mj}$):} The simplest isotropic radial damping in 2D is $-\gamma_{mj} I_2$, parameterized by a positive decay rate $\gamma_{mj} > 0$, where $I_2$ is the $2\times 2$ identity matrix.
\end{enumerate}
Summing these two components gives the $2\times 2$ mode generator block:
\begin{equation}
J=\begin{pmatrix}0&-1\\1&0\end{pmatrix},\qquad
A_{mj}=-\gamma_{mj}I_2+\omega_{mj}J,\quad \gamma_{mj}>0.
\label{eq:generator_block}
\end{equation}

\subsubsection{\textbf{Step 3: Closed-Form Discrete Flow and Absence of Spectral Rescaling}}
A crucial advantage of the block generator $A_{mj}$ in \eqref{eq:generator_block} is that the identity matrix $I_2$ and the rotation matrix $J$ commute ($I_2 J = J I_2 = J$). Since $J^2 = -I_2$, the unit-time flow matrix $B_{mj} = e^{A_{mj}} = e^{-\gamma_{mj} I_2} e^{\omega_{mj} J}$ (Fig.~\ref{fig:architecture}(a)) can be expanded using Taylor's expansion as
\begin{equation}
B_{mj} = e^{A_{mj}} = e^{-\gamma_{mj}}R(\omega_{mj}),\\
\label{eq:block}
\end{equation}
where the unitary rotation matrix
\begin{equation}
R(\omega_{mj})=
\begin{pmatrix}
\cos\omega_{mj}&-\sin\omega_{mj}\\
\sin\omega_{mj}& \cos\omega_{mj}
\end{pmatrix}.
\end{equation}
The complex conjugate eigenvalues of $B_{mj}$ are $e^{-\gamma_{mj} \pm i\omega_{mj}}$. This parameterization cleanly separates operational roles:
\begin{itemize}
\item $\omega_{mj}$ determines the \emph{phase advance per time step} (modal oscillation frequency).
\item $\gamma_{mj}$ determines the \emph{radial contraction rate}, giving a characteristic linear memory e-folding time of $\tau_{mj} = 1/\gamma_{mj}$.
\end{itemize}

In standard ESNs, dense random recurrent matrices are non-normal, and their spectral radius $\rho(W)$ is rescaled post-hoc to enforce contraction ($\rho < 1$). Here $W$ is the recurrent reservoir weight matrix. In our architecture, no post-hoc spectral rescaling is applied or required. Because $B_{mj}$ is a scaled 2D rotation matrix with singular values $e^{-\gamma_{mj}}$, and the similarity transformation $Q_m$ is orthogonal ($Q_m^\top Q_m = I$), the weight matrix $W_m$ remains a \emph{normal matrix} with exact operator $\ell_2$-norm $\|W_m\|_2 = \max_j e^{-\gamma_{mj}} = e^{-\min_j \gamma_{mj}} < 1$ (Fig.~\ref{fig:architecture}(a)). Thus, the sampled decay rates $\gamma_{mj} > 0$ directly certify global echo-state stability by construction. Applying an artificial post-hoc spectral rescaling would rescale all decay rates uniformly, destroying the operational meaning of the sampled decay spectrum $\{\gamma_{mj}\}$ and memory timescales $\{\tau_{mj}\}$.

\subsubsection{\textbf{Step 4: Orthogonal Mode Coupling and Reservoir Assembly}}
If the 25 mode blocks $B_{mj}$ were kept uncoupled, each coordinate pair would remain permanently isolated from input mixing. As depicted in Fig.~\ref{fig:architecture}(a), member $m$ couples coordinates by an orthogonal similarity transformation:
\begin{equation}
\begin{aligned}
W_m &= Q_m \operatorname{blkdiag}(B_{m1},\dots,B_{m,25}) Q_m^\top \in \R^{50 \times 50}, \\
Q_m^\top Q_m &= I_{50},
\end{aligned}
\label{eq:wrec}
\end{equation}
where $Q_m \in \R^{50 \times 50}$ is a random orthogonal matrix obtained via QR decomposition. The transform $Q_m$ thoroughly mixes modal coordinates across all 50 dimensions without altering the pole radii, angles, singular values, or decay rates.

As shown in Fig.~\ref{fig:architecture}(b), each member updates its internal state vector $z_t^{(m)} \in \R^{50}$ driven by external input $x_t \in \R^{d_x}$:
\begin{equation}
z_{t+1}^{(m)}=\tanh\!\left(
W_mz_t^{(m)}+W_{\rm in}^{(m)}x_t+b_m\right),\qquad
z_{0}^{(m)}=0,
\label{eq:update}
\end{equation}
where $z_t^{(m)} \in \R^{50}$. Here,  $W_{\mathrm{in}}^{(m)} \in \R^{50 \times d_x}$, and $b_m \in \R^{50}$ are frozen input projection weights and biases for input $x_t\in \R^{d_x}$.

Fig.~\ref{fig:architecture}(c) details the full ensemble state generation and readout fit: the states of $M=6$ parallel 50-state members ($d_m = 50$) are concatenated into a global 300-dimensional state vector $z_t \in \R^{300}$, with global recurrent block-diagonal operator $\Wrec \in \R^{300 \times 300}$:
\begin{equation}
\begin{aligned}
z_t&=\operatorname{col}_{m=1}^{M}z_t^{(m)} \in \R^{300}, \\
\Wrec&=\operatorname{blkdiag}(W_1,\ldots,W_M) \in \R^{300 \times 300}.
\end{aligned}
\label{eq:ensemble_state}
\end{equation}

After a washout period of $T_{\mathrm{washout}}$ time steps, the global state vectors $z_t \in \mathbb{R}^{300}$ are standardized across the fitting time index set $\mathcal{T}_{\mathrm{fit}}$ according to 
\begin{equation}
\widetilde{z}_t = (z_t - \mu_z) \oslash \sigma_z \in \mathbb{R}^{300},
\end{equation}
where $\mu_z \in \mathbb{R}^{300}$ and $\sigma_z \in \mathbb{R}^{300}$ denote the empirical mean and standard deviation vectors of the reservoir states, respectively, and $\oslash$ represents Hadamard (element-wise) division. The standardized state $\widetilde{z}_t$ is then mapped to a task-dependent feature vector $\phi_t \in \mathbb{R}^{d_\phi}$. Specifically, $\phi_t = \widetilde{z}_t \in \mathbb{R}^{300}$ ($d_\phi = 300$) for linear regression tasks, or $\phi_t = [\widetilde{z}_t \,;\, \widetilde{z}_t^{\odot 2}] \in \mathbb{R}^{600}$ ($d_\phi = 600$) for nonlinearly separable tasks such as delayed XOR, where $\widetilde{z}_t^{\odot 2}$ denotes Hadamard (element-wise) squaring and $[\,\cdot\,;\,\cdot\,]$ represents vertical vector concatenation. Finally, the feature vectors are mapped to target outputs $y_t \in \mathbb{R}^{d_y}$ via regularized linear (ridge) regression:
\begin{equation}
(W_{\mathrm{out}}, c_{\mathrm{out}}) = \arg\min_{W, c} \sum_{t \in \mathcal{T}_{\mathrm{fit}}} \| y_t - W \phi_t - c \|_2^2 + \lambda \| W \|_F^2,
\label{eq:ridge}
\end{equation}
where $W_{\mathrm{out}} \in \mathbb{R}^{d_y \times d_\phi}$ is the readout weight matrix, $c_{\mathrm{out}} \in \mathbb{R}^{d_y}$ is the output bias vector, $\lambda > 0$ is the ridge regularization hyperparameter, $\|\cdot\|_2$ is the Euclidean norm, and $\|\cdot\|_F$ denotes the Frobenius norm. Because the internal state generation parameters remain fixed, fitting $(W_{\mathrm{out}}, c_{\mathrm{out}})$ reduces to a fast, convex closed-form linear solve that requires no backpropagation through time.

\subsection{Multi-Timescale Ensemble}

To provide a matched-feature comparison with standard baselines, all models operate under a ceiling state budget of $d_z = 300$ features. The reported architecture divides this budget into $M=6$ independent parallel members with $d_m = d_z / M = 50$ state coordinates ($25$ mode pairs) per member. 

For member $m$, the base angle, dissipation, and input scales are multiplied by fixed-seed draws $\xi_m^\omega\in[0.5,1.8]$ and $\xi_m^\eta\in[0.5,2.0]$; its input gain is multiplied by $\xi_m^g\in[0.6,1.6]$. Within member $m$, $\omega_{mj}$ is uniform on $[-\xi_m^\omega\omega_{\max},\xi_m^\omega\omega_{\max}]$ and $\gamma_{mj}$ is uniform on $[\xi_m^\eta\eta,3\xi_m^\eta\eta]$. The result is a broad multi-timescale spectrum of damped frequencies across the 150 total mode pairs ($6 \times 25 = 150$), while preserving the exact norm calculation. In particular, every mode obeys
\begin{equation}
\gamma_{mj}\geq \tfrac12\eta
\label{eq:ensemble_floor}
\end{equation}
for the base dissipation parameter $\eta$.

Two mechanism ablations use the same construction. Setting $\omega_j=0$ gives the $D$-only reservoir, a set of rotated decay modes. Setting $\gamma_{mj}=0$ gives the $\Omega$-only reservoir. The latter is nonexpansive before the activation but lacks a uniform strict contraction margin; input-driven saturation may still make its empirical conditional exponent negative.

\subsection{Initialization and Computational Structure}

The implementation follows Fig.~\ref{fig:architecture}. Within each member, the 25 $2\times2$ blocks are assembled into $\operatorname{blkdiag}(B_{m1},\dots,B_{m,25}) \in \R^{50 \times 50}$ and a Gaussian matrix is orthogonalized by QR factorization to obtain $Q_m \in \R^{50 \times 50}$. Input weights $W_{\mathrm{in}}^{(m)} \in \R^{50 \times d_x}$ are drawn independently from $\mathcal U[-g_{\rm in},g_{\rm in}]$ and biases $b_m \in \R^{50}$ from $\mathcal N(0,0.01^2)$. These quantities are frozen after initialization. The member-specific multiplier $\xi_m^g$ is included in $g_{\rm in}$. State features are standardized using statistics from the fitted portion only, and the stored transformation is applied unchanged to test states.

The ensemble structure also yields a major computational benefit. For member dimension $d_m = 50$ and total dimension $d_z = 300$, one recurrent step matrix-vector update costs $O(\sum_{m=1}^6 d_m^2 + d_z d_x) = O(6 \times 50^2 + 300 d_x) = 15,000 + 300 d_x$ operations, rather than $O(d_z^2 + d_z d_x) = 300^2 + 300 d_x = 90,000 + 300 d_x$ operations for a single dense 300-unit reservoir. This block-diagonal structure delivers an exact $6\times$ reduction ($O(d_z^2/M)$) in recurrent matrix multiplication FLOPs. The linear readout remains the dominant trainable component and has $d_y(d_z+1) = 1 \times (300+1) = 301$ trainable parameters for a linear feature map. We report measured pipeline time rather than claiming an asymptotic speedup, because QR construction, BLAS kernels, and ridge fitting contribute differently at this moderate dimension.

\section{Uniform Echo-State Guarantee}
\label{sec:theory}

The construction permits a short global proof that avoids controlling
products of non-normal Jacobians through their individual spectral radii.

\begin{proposition}[Uniform ESP]
\label{prop:esp}
Consider the ensemble recurrence defined by
\eqref{eq:update}-\eqref{eq:ensemble_state}. If
$\gamma_{mj}\geq\underline{\gamma}>0$ for every member and mode, then for
any input sequence and any two initial states $z_0,z'_0$ driven by that
same sequence,
\begin{equation}
\|z_t-z'_t\|_2\leq e^{-t\underline{\gamma}}
\|z_0-z'_0\|_2.
\label{eq:esp_bound}
\end{equation}
Hence the reservoir has the echo-state property. For the ensemble in
Section~\ref{sec:method}, one can take
$\underline{\gamma}=\eta/2$.
\end{proposition}

\begin{proof}
Each rotation in \eqref{eq:block} is orthogonal, so the two singular values
of $B_{mj}$ are both $e^{-\gamma_{mj}}$. Orthogonal similarity and block
concatenation preserve the union of singular values; therefore
\begin{equation}
\|\Wrec\|_2=\max_{m,j} e^{-\gamma_{mj}}
=e^{-\min_{m,j}\gamma_{mj}}\leq e^{-\underline{\gamma}}<1.
\label{eq:norm}
\end{equation}
The componentwise hyperbolic tangent is one-Lipschitz. For two trajectories
with a common input,
\begin{align}
\|z_{t+1}-z'_{t+1}\|_2
&\leq\|\Wrec(z_t-z'_t)\|_2 \nonumber\\
&\leq e^{-\underline{\gamma}}\|z_t-z'_t\|_2.
\end{align}
Iteration gives \eqref{eq:esp_bound}, and
\eqref{eq:ensemble_floor} gives the stated ensemble floor.
\end{proof}

Because the ensemble recurrence is block diagonal, its induced two-norm is
set by the least-damped mode across all members. The prescribed decay range
therefore yields the deterministic bound used above for every admissible
parameter realization. This transparency is specific to the normal
construction: spectral-radius scaling of a dense non-normal ESN does not, in
general, provide the same operator-norm guarantee.

\begin{corollary}[Input-to-state fading memory]
\label{cor:fading}
Let $q=e^{-\underline{\gamma}}$. For two input sequences
$\{x_t\}$ and $\{x'_t\}$, with corresponding states generated by the same
fixed reservoir,
\begin{equation}
\begin{split}
\|z_t-z'_t\|_2
&\leq q^t\|z_0-z'_0\|_2\\
&\quad+\|\Win\|_2\sum_{s=0}^{t-1}
q^{t-1-s}\|x_s-x'_s\|_2.
\end{split}
\label{eq:input_state_bound}
\end{equation}
Consequently, the influence of an input perturbation $k$ steps in the past
is bounded by a factor proportional to $q^k$.
\end{corollary}

\begin{proof}
By definition of the state update equation, two state trajectories $z_t$ and $z'_t$ driven by inputs $x_t$ and $x'_t$ evolve as:
\begin{align*}
z_{t+1} &= \tanh(\Wrec z_t + \Win x_t + b), \\
z'_{t+1} &= \tanh(\Wrec z'_t + \Win x'_t + b).
\end{align*}
Because the componentwise hyperbolic tangent is non-expansive ($1$-Lipschitz continuous with respect to the $\ell_2$-norm), we apply the triangle inequality and operator norm bounds:
\begin{align*}
\|z_{t+1}-z'_{t+1}\|_2 
&\leq \|(\Wrec z_t + \Win x_t + b) \\
&\quad - (\Wrec z'_t + \Win x'_t + b)\|_2 \\
&\leq \|\Wrec(z_t - z'_t)\|_2 + \|\Win(x_t - x'_t)\|_2 \\
&\leq \|\Wrec\|_2 \|z_t - z'_t\|_2 + \|\Win\|_2 \|x_t - x'_t\|_2.
\end{align*}
Recalling from Proposition~\ref{prop:esp} that $\|\Wrec\|_2 \leq e^{-\underline{\gamma}} = q < 1$, this yields the single-step recursive relation:
\begin{equation}
\|z_{t+1}-z'_{t+1}\|_2 \leq q \|z_t-z'_t\|_2 + \|\Win\|_2 \|x_t-x'_t\|_2.
\label{eq:single_step_bound}
\end{equation}
Unrolling relation \eqref{eq:single_step_bound} recursively for $t$ consecutive time steps gives:
\begin{align*}
\|z_t - z'_t\|_2 
&\leq q \Big( q \|z_{t-2} - z'_{t-2}\|_2 \\
&\quad + \|\Win\|_2 \|x_{t-2} - x'_{t-2}\|_2 \Big) \\
&\quad + \|\Win\|_2 \|x_{t-1} - x'_{t-1}\|_2 \\
&\leq q^t \|z_0 - z'_0\|_2 + \|\Win\|_2 \sum_{s=0}^{t-1} q^{t-1-s} \|x_s - x'_s\|_2,
\end{align*}
which completes the derivation of \eqref{eq:input_state_bound}.
\end{proof}

Proposition~\ref{prop:esp} addresses initialization, whereas
Corollary~\ref{cor:fading} quantifies sensitivity to the input history.
Their common factor $q$ makes the design trade-off explicit. Increasing
the smallest decay improves the worst-case forgetting rate but reduces the
longest available linear time constant. Changing $\omega_j$ alters phase
mixing without changing this uniform bound. This separation is an important
theoretical benefit of the parameterization.

For an input-driven trajectory, the Jacobian is
\begin{equation}
J_t=\operatorname{diag}\!\left(\operatorname{sech}^2(a_t)\right)\Wrec,
\qquad a_t=\Wrec z_t+\Win x_t+b.
\label{eq:jacobian}
\end{equation}
We estimate the largest conditional Lyapunov exponent by tangent-vector
iteration through the product $J_{T-1}\cdots J_0$, renormalizing after each
step. This is distinct from averaging $\log\rho(J_t)$.
Proposition~\ref{prop:esp} implies
\begin{equation}
\lambda_{\max}\leq\log\|\Wrec\|_2\leq-\underline{\gamma}.
\label{eq:lyap_bound}
\end{equation}
Activation saturation can make the observed exponent considerably more
negative than this input-uniform upper bound.

The bound is sufficient, not necessary. In particular, the
$\Omega$-only ablation has $\|\Wrec\|_2=1$ and therefore falls outside the
strict global condition, yet its driven trajectory can contract because
$\operatorname{sech}^2(a_t)<1$ away from the activation origin. Conversely,
the operator-norm proof makes no distributional assumption about the input
and remains valid even when the activation operates near its linear
regime.

Also, the normal recurrent construction therefore provides an input-uniform
echo-state guarantee directly from the prescribed decay spectrum. Driven
stability was additionally assessed using a finite-time conditional
Lyapunov exponent and an independent state-separation test; implementation
details are provided in Appendix~\ref{app:conditional_stability}.

\begin{algorithm}[!ht]
\caption{Construction and Training of the Lindblad-Inspired Reservoir}
\label{alg:lindblad_reservoir}
\begin{algorithmic}[1]
\Require Input--target sequence
$\{(\mathbf{x}_t,\mathbf{y}_t)\}_{t=1}^{T}$;
number of members $M$; member dimension $d_m$;
mode parameters $\{\omega_{mj},\gamma_{mj}\}$;
input weights $\mathbf{W}^{(m)}_{\mathrm{in}}$;
biases $\mathbf{b}_m$; washout length $T_{\mathrm{washout}}$;
ridge parameter $\lambda$
\Ensure Fitted readout
$(\mathbf{W}_{\mathrm{out}},\mathbf{c}_{\mathrm{out}})$

\For{$m=1,\ldots,M$}
    \For{$j=1,\ldots,d_m/2$}
        \State Form the damped-rotation mode
        \[
        \mathbf{B}_{mj}
        =
        e^{-\gamma_{mj}}\mathbf{R}(\omega_{mj})
        \]
    \EndFor
    \State Generate an orthogonal matrix $\mathbf{Q}_m$
    by QR factorization
    \State Construct the fixed recurrent operator
    \[
    \mathbf{W}_m
    =
    \mathbf{Q}_m
    \operatorname{blkdiag}
    (\mathbf{B}_{m1},\ldots,\mathbf{B}_{m,d_m/2})
    \mathbf{Q}_m^{\top}
    \]
    \State Initialize $\mathbf{z}^{(m)}_0=\mathbf{0}$
\EndFor

\For{$t=0,\ldots,T-1$}
    \For{$m=1,\ldots,M$}
        \State Update member state
        \[
        \mathbf{z}^{(m)}_{t+1}
        =
        \tanh\!\left(
        \mathbf{W}_m\mathbf{z}^{(m)}_t
        +
        \mathbf{W}^{(m)}_{\mathrm{in}}\mathbf{x}_t
        +
        \mathbf{b}_m
        \right)
        \]
    \EndFor
    \State Concatenate the member states
    \[
    \mathbf{z}_{t+1}
    =
    \operatorname{col}_{m=1}^{M}
    \mathbf{z}^{(m)}_{t+1}
    \]
\EndFor

\State Discard the first $T_{\mathrm{washout}}$ states
\State Estimate $\boldsymbol{\mu}_z$ and
$\boldsymbol{\sigma}_z$ over the fitting index set
$\mathcal{T}_{\mathrm{fit}}$
\For{$t\in\mathcal{T}_{\mathrm{fit}}$}
    \State Standardize the global state
    \[
    \widetilde{\mathbf{z}}_t
    =
    (\mathbf{z}_t-\boldsymbol{\mu}_z)
    \oslash
    \boldsymbol{\sigma}_z
    \]
    \State Form the task-dependent feature vector
    \[
    \boldsymbol{\phi}_t
    =
    \widetilde{\mathbf{z}}_t
    \]
    \Comment{Use
    $[\widetilde{\mathbf{z}}_t;
    \widetilde{\mathbf{z}}_t^{\odot 2}]$
    for delayed XOR}
\EndFor

\State Fit the ridge readout
\[
(\mathbf{W}_{\mathrm{out}},\mathbf{c}_{\mathrm{out}})
=
\arg\min_{\mathbf{W},\mathbf{c}}
\sum_{t\in\mathcal{T}_{\mathrm{fit}}}
\left\|
\mathbf{y}_t-\mathbf{W}\boldsymbol{\phi}_t-\mathbf{c}
\right\|_2^2
+
\lambda\|\mathbf{W}\|_F^2
\]

\State \Return
$(\mathbf{W}_{\mathrm{out}},\mathbf{c}_{\mathrm{out}})$
\end{algorithmic}
\end{algorithm}

\section{Experimental Design}
\label{sec:protocol}

The evaluation separates three issues that are often conflated in reservoir
comparisons. Predictive benchmarks test usefulness, the dissipation sweep
and ablation test mechanism, and the timing audit tests the cost of reaching
the fitted predictor. The same 300-state budget and seed-aligned data are
used wherever the model classes permit a direct comparison. Algorithm~\ref{alg:lindblad_reservoir} summarizes the complete training
procedure. The recurrent operators are constructed analytically from the
prescribed rotation and decay parameters and remain fixed throughout
training; only the linear readout is optimized.

\begin{table*}[!ht]
\caption{Evaluation design. All splits are chronological. The listed
length is the number of input-target pairs before the 55/20/25\% split;
each reported reservoir result uses ten aligned task and initialization
seeds.}
\label{tab:protocol}
\centering\footnotesize\setlength{\tabcolsep}{5pt}
\begin{tabular}{lcccccl}
\toprule
Task & Input dim. & Length & Horizon/delay & Output & Feature map &
Primary metric\\
\midrule
Linear memory & 1 & 5000 & $\tau=1{:}150$ & continuous & linear &
$\sum_\tau\mathrm{Corr}^2$\\
NARMA-10 & 1 & 6000 & one step & continuous & linear & NRMSE\\
Bounded NARMA-20 & 1 & 6000 & one step & continuous & linear & NRMSE\\
Mackey-Glass-17 & 1 & 6000 & $k=1,12,84$ & continuous & linear & NRMSE\\
Lorenz-63 & 1 & 6000 & $k=50$ & continuous & linear & NRMSE\\
Delayed XOR & 2 & 8000 & $\tau=10$ & binary & linear+square & accuracy\\
UCI Air Quality & 8 & 9357 & same time & continuous & linear & NRMSE\\
\bottomrule
\end{tabular}
\end{table*}

\subsection{Models and Validation}

The recurrent models are a standard dense ESN~\cite{jaeger2001}, a
correctly implemented leaky ESN~\cite{jaeger2007leaky}, a two-layer Deep
ESN~\cite{gallicchio2017}, a frozen orthogonal reservoir,
CRJ~\cite{rodan2012}, and the proposed model. All expose 300 state features
to the readout. The Deep ESN uses two 150-unit layers and concatenates
their states. NG-RC uses linear and second-order monomials of delayed inputs
under the same 300-feature ceiling. The compact GRU uses 32 hidden units on
the two regression comparisons and 64 on XOR.

For clarity, the standard and leaky ESNs use
\begin{align}
\widetilde z_{t+1}&=\tanh(Wz_t+\Win x_t+b), \nonumber\\
z_{t+1}&=(1-\alpha)z_t+\alpha\widetilde z_{t+1},
\label{eq:leaky}
\end{align}
with $\alpha=1$ for the standard ESN.\@ The dense Gaussian $W$ is rescaled
once to the selected spectral radius; the leak is not folded into $W$ and
then applied a second time. The orthogonal baseline uses $W=\rho Q$.
CRJ uses a directed cycle with regularly spaced bidirectional jumps,
followed by spectral-radius rescaling. Deep ESN updates two 150-unit leaky
layers sequentially and concatenates both layer states, so its readout
feature count remains 300.

For NG-RC, let
\begin{equation}
v_t=[x_t^\top,x_{t-1}^\top,\ldots,x_{t-L+1}^\top]^\top.
\end{equation}
Its fixed feature vector contains $v_t$ and the upper-triangular products
$v_{t,i}v_{t,j}$ for $i\leq j$. If $p=Ld_x$, the resulting dimension is
$p+p(p+1)/2$. We clip $L$ so this number does not exceed 300. This makes
the comparison one of matched features rather than matched internal
operations: NG-RC has no recurrent state transition.

Hyperparameters are selected once for each task and model on a dedicated
validation realization and then held fixed for ten evaluation seeds. This
avoids choosing a different configuration for every reporting seed.
The chronological split is 55\% training, 20\% validation, and 25\% test,
with a 100-step washout for reservoirs. The validation grids are
$\rho\in\{0.985,0.995\}$, input gain in $\{1.0,1.5\}$, leak in
$\{0.5,1.0\}$, $\omega_{\max}\in\{2,4\}$,
$\eta\in\{0.003,0.01,0.03\}$, and
$\lambda\in\{10^{-8},10^{-6},10^{-4},10^{-2}\}$. NG-RC delay depths are
$\{5,10,15,20\}$, clipped if necessary to respect the feature budget.
The selected settings and all seed-level values are included with the
reproducibility scripts.

The validation realization is not reused as a reporting seed. At each of
the ten evaluation indices, every model receives the same freshly
generated task sequence and its own deterministically paired reservoir
initialization. Thus a paired comparison changes the architecture while
holding the data realization and seed index aligned. After model selection,
the training and validation portions are combined to fit the final
readout; the test portion remains untouched. Input and state normalization
statistics are fitted only from samples available to the corresponding
fitting stage and never from the test segment. For memory capacity, the
validation objective is a fixed
10-step recall proxy rather than the reported sum over 150 delays.

The GRU uses chronological windows and normalization fitted only to the
training segment. For bounded NARMA-20 and Mackey-Glass it has 32 units,
100-step windows, Adam learning rate $3\times10^{-3}$, at most 80 epochs,
and early-stopping patience 10. For XOR it has 64 units, 20-step windows,
learning rate $10^{-2}$, and 100 epochs. These are compact trained
references, not an exhaustive neural architecture search.

\begin{table*}[t]
\caption{MC stands for linear memory capacity. N10 and BN20 denote standard NARMA-10 and bounded NARMA-20; MG12/MG84 are Mackey-Glass horizons; L63 is Lorenz-63 at 50 steps. Higher is better for MC and XOR; lower is better otherwise. Best values are bold. The GRU is an end-to-end-trained reference and is reported only on its three prespecified tasks.}
\label{tab:performance}
\centering\scriptsize\setlength{\tabcolsep}{3.0pt}
\begin{tabular}{lrrrrrrr}
\toprule
Model & MC & N10 & BN20 & MG12 & MG84 & L63 & XOR (\%) \\
\midrule
ESN & 14.98$\pm$0.42 & 0.299$\pm$0.026 & 0.494$\pm$0.039 & 0.031$\pm$0.007 & 0.204$\pm$0.021 & 0.661$\pm$0.227 & 51.85$\pm$1.00 \\
Leaky ESN & 14.98$\pm$0.42 & 0.299$\pm$0.026 & 0.494$\pm$0.039 & \textbf{0.002$\pm$0.000} & \textbf{0.088$\pm$0.012} & 0.661$\pm$0.227 & 64.71$\pm$1.11 \\
Deep ESN & 14.64$\pm$0.65 & 0.331$\pm$0.025 & 0.536$\pm$0.033 & 0.003$\pm$0.000 & 0.093$\pm$0.016 & 0.860$\pm$0.480 & 62.59$\pm$1.11 \\
NG-RC & 14.11$\pm$0.03 & \textbf{0.172$\pm$0.016} & 0.785$\pm$0.006 & 0.082$\pm$0.004 & 0.535$\pm$0.022 & 1.017$\pm$0.055 & \textbf{100.00$\pm$0.00} \\
Orthogonal & \textbf{17.73$\pm$0.51} & 0.346$\pm$0.017 & 0.442$\pm$0.020 & 0.037$\pm$0.005 & 0.224$\pm$0.023 & 0.950$\pm$0.791 & 53.77$\pm$1.10 \\
CRJ & 10.81$\pm$0.14 & 0.386$\pm$0.023 & 0.614$\pm$0.032 & 0.161$\pm$0.011 & 0.395$\pm$0.015 & 0.834$\pm$0.316 & 49.97$\pm$0.86 \\
\textbf{Lindblad-inspired} & 17.28$\pm$1.22 & 0.296$\pm$0.039 & \textbf{0.419$\pm$0.025} & 0.027$\pm$0.006 & 0.174$\pm$0.028 & \textbf{0.574$\pm$0.113} & 55.75$\pm$3.68 \\
Compact GRU & -- & -- & 0.643$\pm$0.050 & -- & 0.167$\pm$0.026 & -- & \textbf{100.00$\pm$0.00} \\
\bottomrule
\end{tabular}
\end{table*}

\begin{table*}[t]
\caption{Holm-adjusted two-sided Wilcoxon signed-rank p-values comparing the proposed model with each reservoir baseline. Tests use ten aligned evaluation indices, with the same task realization used across models at each index. Holm correction is applied across the six baseline comparisons within each task.}
\label{tab:significance}
\centering\footnotesize\setlength{\tabcolsep}{6pt}
\begin{tabular}{lrrrrrr}
\toprule
Task & ESN & Leaky ESN & Deep ESN & NG-RC & Orthogonal & CRJ \\
\midrule
MC & \textbf{0.0117} & \textbf{0.0117} & \textbf{0.0117} & \textbf{0.0117} & 0.3750 & \textbf{0.0117} \\
N10 & 1.0000 & 1.0000 & 0.1113 & \textbf{0.0117} & \textbf{0.0117} & \textbf{0.0117} \\
BN20 & \textbf{0.0117} & \textbf{0.0117} & \textbf{0.0117} & \textbf{0.0117} & \textbf{0.0371} & \textbf{0.0117} \\
MG-12 & 0.4316 & \textbf{0.0117} & \textbf{0.0117} & \textbf{0.0117} & \textbf{0.0117} & \textbf{0.0117} \\
MG-84 & \textbf{0.0117} & \textbf{0.0117} & \textbf{0.0117} & \textbf{0.0117} & \textbf{0.0117} & \textbf{0.0117} \\
L63 & 0.6445 & 0.6445 & 0.5801 & \textbf{0.0117} & 0.2578 & \textbf{0.0117} \\
XOR-10 & \textbf{0.0156} & \textbf{0.0156} & \textbf{0.0156} & \textbf{0.0117} & 0.2324 & \textbf{0.0117} \\
\bottomrule
\multicolumn{7}{l}{\scriptsize Bold: $p_{\rm Holm}<0.05$; direction of the difference is given by Table~\ref{tab:performance}.}\\
\end{tabular}
\end{table*}

\subsection{Tasks and Metrics}

Linear memory capacity uses 150 delays and
$MC_\tau=\operatorname{Corr}^2(x_{t-\tau},\hat x_{t-\tau})$ with
$MC=\sum_\tau MC_\tau$~\cite{jaeger2001}. Standard NARMA-10 follows the
recurrence~\cite{atiya2000narma}
\begin{equation}
y_{t+1}=0.3y_t+0.05y_t\sum_{i=t-9}^{t}y_i
+1.5u_{t-9}u_t+0.1,
\label{eq:n10}
\end{equation}
where $u_t\sim\mathcal U[0,0.5]$. Because unsaturated order-20 variants can
diverge under commonly used coefficients, the longer task is explicitly
defined as \emph{bounded NARMA-20}:
\begin{equation}
\begin{aligned}
y_{t+1}=\tanh\!\bigg(&0.3y_t+0.05y_t\sum_{i=t-19}^t y_i\\
&+1.5u_{t-19}u_t+0.01\bigg).
\end{aligned}
\label{eq:bn20}
\end{equation}
The bounded designation is retained throughout; it is not presented as the
unsaturated canonical recurrence.

Mackey-Glass with delay 17~\cite{mackeyglass1977} is generated from
\begin{equation}
\dot x(t)=\frac{0.2x(t-17)}{1+x(t-17)^{10}}-0.1x(t)
\label{eq:mg}
\end{equation}
using a unit integration step and 1000 discarded transient steps. Direct
targets are $x_{t+k}$ for $k\in\{1,12,84\}$; predictions are not recursively
rolled out. Lorenz-63~\cite{lorenz1963} uses
\begin{equation}
\dot x=10(y-x),\quad
\dot y=x(28-z)-y,\quad
\dot z=xy-\tfrac83z,
\label{eq:lorenz}
\end{equation}
integrated by fourth-order Runge-Kutta with step 0.02 after a 2000-step
transient. Its $x$ component is predicted 50 steps ahead. For delayed XOR,
independent Bernoulli streams $a_t,b_t$ are mapped to $\{-1,1\}$ inputs and
the target is
\begin{equation}
y_t=a_{t-10}\operatorname{~XOR~}b_t.
\label{eq:xor}
\end{equation}
This task probes the combination of a delayed memory and a multiplicative
logical interaction.

For an external-validity check, we use the UCI Air Quality
series~\cite{vito2008airquality}, which contains hourly responses from five
metal-oxide sensors and reference gas measurements collected in an Italian
city. The task estimates reference CO concentration from the five sensor
responses plus temperature, relative humidity, and absolute humidity. The
provided missing-value marker $-200$ is converted to missing data. Input
channels are filled only from the most recent past observation, while
hours without a valid reference CO value are excluded from readout fitting
and scoring. The split remains chronological, and scaling uses the training
portion only. This is temporal sensor calibration under drift, not a
future-forecasting claim. It is worth mentioning that recent application-oriented developments have combined multilayer reservoir
architectures, adaptive plasticity, and feature decomposition for air-quality
forecasting \cite{xu2026enhanced}. Such approaches optimize the reservoir
for a specific forecasting domain, whereas the present work studies a
closed-form recurrent construction with analytically separated rotation,
decay, and stability controls.

Regression uses
$\mathrm{NRMSE}=\sqrt{\mathrm{MSE}/\operatorname{Var}(y)}$; XOR uses
accuracy. Effective rank is the exponentiated entropy of the normalized
singular values of the centered state matrix~\cite{roy2007}. We report mean and population standard deviation over ten seeds. Pairwise
reservoir comparisons use two-sided Wilcoxon signed-rank tests on aligned
seed indices, with Holm correction within each task. NG-RC is included in the paired tests because, at each seed index, it is evaluated on the same task realization as the reservoir models. The GRU is reported only as a separately trained reference and is excluded from the paired statistical tests.
Population rather than sample standard deviation is used because the ten
seeds are the complete prespecified evaluation set. No test result is used
to retune a model. Evaluation design is presented in Table~\ref{tab:protocol}.

\section{Results}
\label{sec:results}

\subsection{Performance Across Tasks}

The results answer the evaluation questions in the same order as the
protocol: predictive scope first, then external calibration and cost, and
finally the rotation-dissipation mechanism. Table~\ref{tab:performance}
shows a task-dependent ordering rather than a universal winner. On memory
capacity, the proposed model reaches
$17.28\pm1.22$, statistically tied with the orthogonal reservoir's
$17.73\pm0.51$ (Holm-adjusted $p=0.375$) and above the ESN, leaky ESN,
Deep ESN, NG-RC, and CRJ ($p_{\rm Holm}=0.0117$ for each). This is an
approximately 15\% increase over the standard ESN and is consistent with
Fig.~\ref{fig:benchmarks}(a): rotation-dominated modes preserve a longer
linear-memory tail, while CRJ decays earliest.

On standard NARMA-10, NG-RC is best at $0.172\pm0.016$ because its selected
20-delay polynomial features directly represent the required interactions.
The proposed model obtains $0.296\pm0.039$, essentially the same as the
ESN's $0.299\pm0.026$ ($p_{\rm Holm}=1.0$). Its advantage over Deep ESN
also does not survive correction ($p_{\rm Holm}=0.111$), whereas its lower
error relative to the orthogonal and CRJ baselines is significant
($p_{\rm Holm}=0.0117$). The result therefore supports parity
with a standard ESN on NARMA-10, which need not be considered as superiority.

For bounded NARMA-20, the proposed reservoir
achieves $0.419\pm0.025$, lower than every other reservoir. The improvement
is significant against ESN, leaky ESN, Deep ESN, NG-RC, and CRJ
($p_{\rm Holm}=0.0117$) and against the orthogonal reservoir
($p_{\rm Holm}=0.0371$). Relative to the strongest competing reservoir,
orthogonal RC at $0.442\pm0.020$, the mean NRMSE is lower by 5.2\%.
NG-RC's validation-selected five-delay representation is too short for this
recurrence, while the compact GRU at $0.643\pm0.050$ does not recover the
long nonlinear dependence under its prespecified small configuration.

\begin{figure*}[t]
\centering
\includegraphics[width=0.98\textwidth]{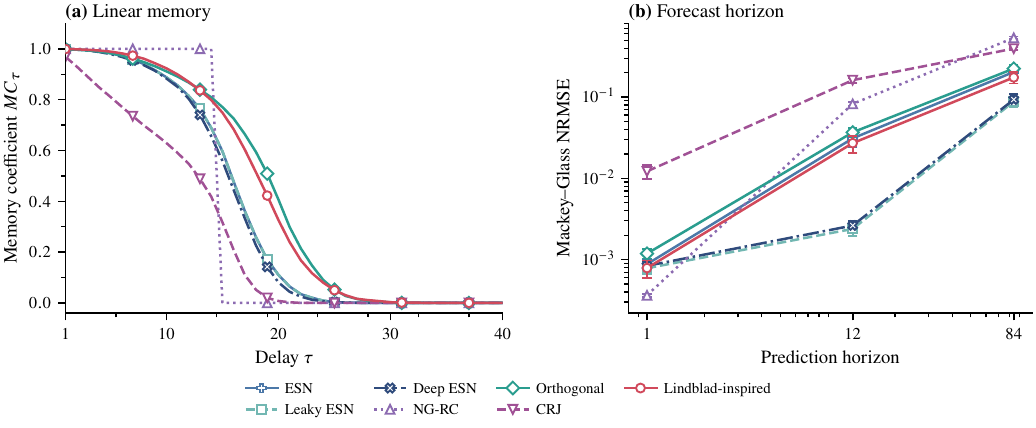}
\caption{Benchmark diagnostics. (a) Delay-wise linear memory. The proposed
model and the orthogonal reservoir retain the longest recurrent memory
tails; NG-RC retains inputs explicitly over its selected finite window.
(b) Mackey-Glass error versus forecast horizon (log scales). Leaky and
Deep ESNs are strongly matched to this smooth signal, while the proposed
model ensemble is intermediate.}
\label{fig:benchmarks}
\end{figure*}

\begin{figure*}[!ht]
\centering
\includegraphics[width=0.98\textwidth]{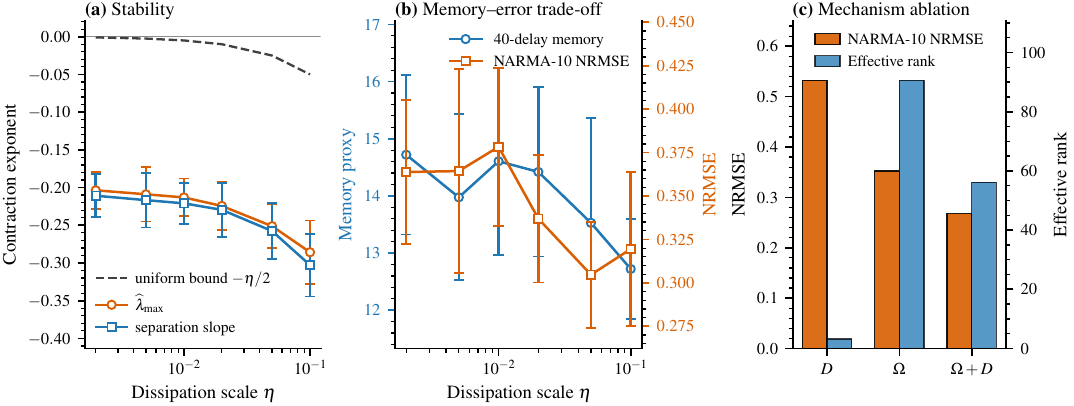}
\caption{Mechanism diagnostics. (a) Conditional Lyapunov exponent,
same-reservoir trajectory-separation slope, and the uniform ensemble bound
$-\eta/2$. (b) Forty-delay memory and NARMA-10 error (c) Mechanism ablation on NARMA-10.}
\label{fig:mechanism}
\end{figure*}

The performance of the proposed model for Mackey-Glass is not the best. At $k=12$ and $84$, the
leaky ESN obtains NRMSE $0.002$ and $0.088$, and the Deep ESN obtains $0.003$
and $0.093$, whereas the proposed model obtains $0.027$ and $0.174$.
The proposed model still improves on the standard ESN at $k=84$
($p_{\rm Holm}=0.0117$), but loses significantly to both leaky variants at
both reported nontrivial horizons. A tunable
leak acts as a low-pass temporal filter well suited to this smooth
delay-differential signal. The proposed ensemble distributes its modes
over a broader range of frequencies.

On Lorenz-63, the proposed model has the lowest mean NRMSE,
$0.574\pm0.113$, compared with $0.661\pm0.227$ for ESN and leaky ESN.\@
The difference is not significant after Holm correction
($p_{\rm Holm}=0.645$), although the improvement over CRJ is significant
($p_{\rm Holm}=0.0117$). We therefore describe the model as competitive,
not generally superior. The smaller standard deviation is encouraging but
is not, by itself, a robustness guarantee. On
delayed XOR, NG-RC and the GRU reach 100\% by exposing or learning the
required multiplicative interaction. The proposed linear-plus-square
readout reaches only $55.75\%$ and is below the leaky and Deep ESNs. This
result prevents an overly broad claim that damped rotational dynamics are
optimal for every nonlinear memory operation.

Table~\ref{tab:significance} makes the statistical scope explicit. A small
$p$-value does not imply that the proposed model is better: direction must
be read together with Table~\ref{tab:performance} as described previously. The
differences from NG-RC on NARMA-10 and XOR and from leaky/Deep ESN on
Mackey-Glass are significant in favor of those baselines. Conversely,
bounded NARMA-20 is the only predictive task on which the proposed model is
significantly better than every reservoir comparator. Memory capacity is a
tie with the strongest orthogonal baseline, and Lorenz-63 remains
inconclusive against the dense, leaky, deep, and orthogonal alternatives. However, in all cases, the proposed model remains competitive with other benchmark scenarios.

\subsection{Real-Sensor Calibration}

\begin{table}[t]
\caption{UCI Air Quality temporal CO calibration from five metal-oxide sensors and three weather variables.}
\label{tab:airquality}
\centering\footnotesize\setlength{\tabcolsep}{4pt}
\begin{tabular}{lrr}
\toprule
Model & NRMSE & $p_{\rm Holm}$ \\
\midrule
Training-mean predictor & 1.000 & -- \\
Static ridge & 0.531 & -- \\
ESN & 0.607$\pm$0.018 & 1.0000 \\
Leaky ESN & 0.607$\pm$0.018 & 1.0000 \\
Deep ESN & 0.675$\pm$0.049 & 0.0195 \\
NG-RC & 0.625$\pm$0.000 & 0.0117 \\
Orthogonal & 0.615$\pm$0.024 & 0.6973 \\
CRJ & 0.667$\pm$0.026 & 0.0195 \\
\textbf{Lindblad-inspired} & 0.603$\pm$0.016 & -- \\
\bottomrule
\end{tabular}
\end{table}

Table~\ref{tab:airquality} adds a deliberately different applied check with air quality data.
The proposed reservoir obtains $0.603\pm0.016$, the lowest mean NRMSE among
the recurrent and NG-RC models. It is statistically indistinguishable from
the ESN, leaky ESN, and orthogonal reservoir after correction, but is lower
than Deep ESN and CRJ.\@ More importantly, static ridge regression on the
current sensor vector obtains 0.531 and outperforms every other temporal model.

\subsection{Performance-Training-Cost Trade-off}

\begin{table}[t]
\caption{Training-cost comparison on bounded NARMA-20. Reservoir time includes construction, state collection, and ridge fit; GRU time is end-to-end optimization. Medians use five timing runs for reservoirs and ten training seeds for the GRU.}
\label{tab:cost}
\centering\footnotesize\setlength{\tabcolsep}{3pt}
\begin{tabular}{lrrc}
\toprule
Model & Trainable & Time (s) & BPTT \\
\midrule
ESN & 301 & 0.16 & No \\
Leaky ESN & 301 & 0.17 & No \\
Deep ESN & 301 & 0.20 & No \\
NG-RC & 21 & 0.03 & No \\
Orthogonal & 301 & 0.10 & No \\
CRJ & 301 & 0.15 & No \\
Lindblad-inspired & 301 & 0.16 & No \\
Compact GRU & 3393 & 52.31 & Yes \\
\bottomrule
\end{tabular}
\end{table}

Table~\ref{tab:cost} compares the computational cost of all models using bounded NARMA-20 as a common benchmark. The evaluation was done on an Apple Silicon M2-powered MacBook Pro laptop. Reservoir time includes matrix construction, state collection,
and ridge fitting, while excluding the one-time validation search for all
models. The proposed framework completes the full pipeline in a median
0.16~s and trains only 301 readout parameters. By comparison, the 32-unit
GRU optimizes 3393 parameters through BPTT and requires 52.31~s on the same
machine, corresponding to an approximately 317-fold increase in training
time. Although absolute timings depend on hardware and implementation, this
large gap highlights the computational competitiveness of the proposed
fixed-dynamics approach: it achieves task-relevant performance with a
lightweight convex readout fit and no recurrent-weight optimization.

This efficiency is obtained without making a universal accuracy claim. The
GRU and NG-RC remain stronger on delayed XOR, and the GRU is marginally
better at Mackey-Glass $k=84$. In contrast, the proposed framework is more
accurate on bounded NARMA-20 while retaining a substantially lower training
cost. The resulting comparison therefore places the method at a favorable
accuracy-efficiency operating point rather than as a uniformly dominant
predictor.

\subsection{Mechanism and Stability}

The convergence experiment propagates 15 distinct initial conditions through
the same fixed reservoir under a common input sequence. Any reduction in
trajectory separation therefore reflects forgetting of the initial state,
rather than differences in the recurrent matrix. Figure~\ref{fig:mechanism}(a)
summarizes the dissipation sweep. Over
$\eta\in[0.002,0.1]$, the Jacobian-product exponent becomes more negative,
decreasing from $-0.204$ to $-0.286$, and closely tracks the directly fitted
trajectory-separation slope. Both empirical contraction measures lie below
the conservative uniform bound $-\eta/2$, consistent with additional
contraction induced by activation saturation. The conditional-exponent and state-separation estimates provide consistent
empirical evidence of driven-state contraction; a brief discussion is given in
Appendix~\ref{app:conditional_stability}.

Increasing dissipation weakens retained memory: the 40-delay memory proxy
decreases from $14.72$ to $12.72$. Prediction performance, however, is
nonmonotonic. The NARMA-10 NRMSE decreases from $0.364$ at
$\eta=0.002$ to $0.305$ at $\eta=0.05$, before increasing to $0.319$ at
$\eta=0.1$. Figure~\ref{fig:mechanism}(b) therefore reveals a trade-off
between stability margin, memory retention, and nonlinear prediction,
rather than a monotonic benefit from increasing dissipation.

The ablation in Fig.~\ref{fig:mechanism}(c) clarifies the distinct roles of
rotation and dissipation. The $\Omega$-only model produces the richest state
representation, with the highest effective rank
($90.66\pm8.42$), but its NARMA-10 error remains
$0.352\pm0.029$. In contrast, the $D$-only model is stable but nearly
degenerate, with effective rank $3.23\pm0.27$ and NRMSE
$0.533\pm0.014$. Combining the two yields an intermediate effective rank
of $56.17\pm3.60$ and the lowest error, $0.268\pm0.023$. Thus,
$\Omega+D$ does not maximize state-space rank. Instead, controlled
dissipation regularizes the highly diverse rotational features into a more
predictive fading-memory representation. For the present example, the combined $\Omega+D$ dynamics are preferable:
rotation supplies a diverse transient representation, while dissipation
suppresses irrelevant history and improves predictive conditioning.

\section{Discussion}
\label{sec:discussion}

\subsection{What the Results Establish}

The theoretical and empirical results support two complementary conclusions.
At the architectural level, the contribution is precise and general. Each
recurrent mode is parameterized by an explicit rotation angle and decay rate,
so phase mixing and forgetting are controlled separately. Because the global
recurrent operator remains normal, its singular values are available in
closed form, and the smallest modal decay rate yields an input-uniform
echo-state-property margin before any task-specific training is performed.
This provides a more transparent stability design than post-hoc
spectral-radius rescaling of a generic non-normal matrix, where the rescaling
parameter does not directly identify the underlying forgetting timescale.

At the predictive level, the evidence is deliberately more qualified. The
proposed reservoir provides a competitive operating point across linear
memory, nonlinear temporal dependence, and chaotic forecasting, with its
clearest advantage on bounded NARMA-20. Its strengths arise from combining
broad phase diversity with controlled contraction, rather than from imposing
a narrow task-specific structure. The same generality can become a
disadvantage when a stronger inductive bias is available: leaky reservoirs
are better matched to the smooth low-pass structure of Mackey-Glass,
NG-RC directly exposes finite-delay polynomial interactions, and an
end-to-end-trained GRU can learn the delayed XOR relation. The proposed
architecture should therefore not be interpreted as a universally superior
reservoir. Its contribution is a stable, analyzable, and computationally
efficient recurrent design whose memory and mixing properties can be
reasoned about mode by mode, and whose accuracy-stability trade-off can be
controlled through physically interpretable parameters.

\subsection{Mechanistic Interpretation}

The stability result establishes that initial-state differences are
forgotten under the stated decay condition, but it does not determine which
forgetting rate is best for a given task. The ablation clarifies this
computational role. The rotation-only model generates a high-dimensional
state representation, but without radial decay it can preserve irrelevant
history. The decay-only model forgets reliably but produces a low-rank
representation. Their combination lowers effective rank relative to pure
rotation while improving NARMA prediction. Useful state geometry therefore
requires neither maximal rank nor maximal contraction, but an appropriate
balance between separating relevant input histories and suppressing
differences beyond the target's effective memory.

The dissipation sweep makes this balance explicit. The guaranteed rate
$-\eta/2$ varies monotonically with $\eta$, whereas prediction error is
nonmonotonic because the formal bound does not account for the input
distribution, activation saturation, or task structure. In practice,
$\eta$ should be interpreted first as a control on the stability margin and
then selected by validation for the target problem. The parameter
$\omega_{\max}$ separately controls phase diversity. Delayed XOR instead
favours weaker decay, consistent with the need to retain a specific delayed
bit.

The multi-timescale organization is another important construction worth highlighting. It adds a second level of structure beyond the
within-member separation of rotation and decay. By combining members with
different decay profiles, the reservoir exposes a spectrum of effective
memory horizons to the same readout. Faster members emphasize recent
transients, whereas slower members retain information over longer intervals.
The resulting feature bank therefore supports tasks whose relevant temporal
dependencies are not known a priori and may occur at more than one scale.
Importantly, this benefit is obtained without training the recurrent
operators or sacrificing the member-wise stability guarantee.

\subsection{When to Use the Architecture}

The results indicate several practical regimes. When a task requires
substantial recurrent memory and nonlinear temporal processing, but
training the recurrent weights is undesirable or infeasible, a spectrum of
damped rotational modes provides a reasonable fixed-dynamics option. The
memory-capacity and bounded-NARMA results support this use case.

When the target is smooth and predominantly low-pass, a conventional leaky
ESN may be both simpler and more accurate, as observed for Mackey-Glass.
When the relevant interaction is known to lie within a short finite delay
window, explicit NG-RC features can be more compact and more effective. The
delayed-XOR result illustrates this case. The air-quality experiment further
shows that recurrence may be unnecessary when the target is explained
mainly by an instantaneous cross-sectional mapping.

The GRU comparison should be interpreted in the same way. The proposed
model is not intended to replace task-specific recurrent optimization. It
offers a different operating point: a fixed and stable state transition, a
convex readout fit, and a stability margin known before supervised training.
This can be advantageous when training cost, reproducibility, or physical
implementation is important. Conversely, when sufficient labeled data and
optimization resources are available, a trained gated model may learn
interactions that are not exposed by a fixed reservoir with a linear
readout.

\subsection{Relation to Lindblad Dynamics}

The construction uses the exact exponential of damped rotational
generators, preserving the continuous-time motivation without introducing a
discretization approximation. The model remains, however,
Lindblad-inspired rather than a realization of a physical GKSL channel.
Physical GKSL generators satisfy additional algebraic constraints associated
with complete positivity and trace preservation. We do not establish that
the orthogonally mixed blocks in \eqref{eq:wrec} satisfy these constraints
for any specific quantum system. Such conditions would be necessary for a
claim of generator-faithful quantum simulation or direct quantum-hardware
realization, but they are not required for the classical
structured-reservoir problem considered here.

The analogy is nevertheless operational. The generator parameters remain
directly identifiable after discretization as pole angles and radii, and
the exact exponential makes the stability margin traceable to the decay
rates. At the same time, density-matrix positivity, trace preservation,
measurement back-action, and Hilbert-space scaling are not part of the
algorithm. The appropriate description is therefore a classical normal
reservoir motivated by a rotation-dissipation decomposition.

\section{Conclusion}
\label{sec:conclusion}

This work develops a classical structured reservoir from an explicit
rotation-dissipation decomposition. The multi-timescale construction combines reservoir members with distinct decay rates, enabling the readout to access both short-lived transients and longer temporal dependencies within a single feature representation. Exact exponentiation assigns each mode
an independently interpretable angle and decay rate, while orthogonal
similarity mixing preserves the modal spectrum and normality. Positive decay
therefore yields the input-uniform bound
$\|\Wrec\|_2\leq e^{-\underline{\gamma}}<1$, linking the recurrent
parameterization directly to both a characteristic memory timescale and a
formal echo-state margin without post-hoc spectral rescaling.

The empirical results support a task-dependent rather than universal
advantage. The proposed six-member reservoir is competitive in linear
memory and Lorenz prediction, achieves the best bounded NARMA-20 performance
among the fixed-reservoir baselines, and retains the low training cost of a
convex readout fit. At the same time, leaky and Deep ESNs are better matched
to smooth Mackey-Glass dynamics, NG-RC and the GRU are stronger on delayed
XOR, and static ridge regression is preferable when the target is primarily
instantaneous, as in the air-quality calibration task. The mechanism
analysis further shows that the useful regime is not obtained by maximizing
effective rank or contraction alone, but by combining rotational diversity
with sufficient dissipative forgetting.

The main contribution is therefore a stable and interpretable recurrent
design whose mixing and memory properties can be controlled mode by mode.
A natural next step is to retain this separation while adapting the decay
spectrum to data and, separately, to investigate generator constructions
that satisfy the constraints of physically realizable open-system dynamics.

\section*{Acknowledgment}
During the preparation of this manuscript, the authors used ChatGPT (GPT-5.5,
OpenAI) and Grammarly to assist with language editing, sentence refinement, and
clarity improvements. All scientific content, formal arguments, mathematical
constructions, and theoretical claims were developed entirely by the authors.
The authors reviewed and take full responsibility for all content in this
manuscript.

\appendices
\section{Additional Stability and Geometry Details}
\label{app:stability}

\subsection{Conditional Exponent and Separation Estimate}
\label{app:conditional_stability}

Starting from a unit tangent vector $v_0$, the largest conditional exponent
is estimated along the driven trajectory through repeated Jacobian
propagation and renormalization:
\begin{align}
\widetilde v_{t+1}&=J_t v_t,\qquad
v_{t+1}=\frac{\widetilde v_{t+1}}
{\|\widetilde v_{t+1}\|_2},\\
\widehat\lambda_{\max}
&=\frac{1}{T-T_{\rm w}}
\sum_{t=T_{\rm w}}^{T-1}
\log \|\widetilde v_{t+1}\|_2.
\label{eq:lyap_estimator}
\end{align}
Each ensemble member is propagated separately, and the largest member-wise
estimate is reported. Because the estimator follows ordered products of
trajectory Jacobians, it accounts for changes in their dominant singular
directions over time.

As an independent empirical check, two random initial states are evolved
through the same fixed reservoir under the same input sequence. Defining
$d_t=\|z_t-z'_t\|_2$, a linear fit is applied to $\log d_t$ over the initial
decay regime, restricted to $d_t>10^{-10}$ and the first 300 steps. The
negative fitted slope is reported as the separation rate. Its agreement with
$-\widehat\lambda_{\max}$ in Fig.~\ref{fig:mechanism}(a) provides a numerical
consistency check rather than an additional proof of the echo-state
property.

\subsection{Complexity and Timing Decomposition}

\begin{table*}[!ht]
\caption{Bounded-NARMA-20 timing decomposition in milliseconds. Reservoir entries are medians over five repetitions. A dash indicates that the end-to-end GRU timing is not separable into fixed-state collection and closed-form readout fitting.}
\label{tab:costbreakdown}
\centering\footnotesize\setlength{\tabcolsep}{7pt}
\begin{tabular}{lrrrrr}
\toprule
Model & Construction & State/features & Readout & Total & Trainable \\
\midrule
ESN & 39.1 & 108.6 & 9.8 & 159.2 & 301 \\
Leaky ESN & 38.2 & 117.1 & 9.5 & 165.6 & 301 \\
Deep ESN & 15.4 & 169.4 & 11.0 & 204.0 & 301 \\
NG-RC & 0.1 & 28.7 & 0.6 & 30.7 & 21 \\
Orthogonal & 4.3 & 86.3 & 10.5 & 102.2 & 301 \\
CRJ & 55.8 & 81.7 & 10.1 & 150.6 & 301 \\
Lindblad-inspired & 3.6 & 154.6 & 9.9 & 164.9 & 301 \\
Compact GRU & -- & -- & -- & 52314.3 & 3393 \\
\bottomrule
\end{tabular}
\end{table*}

Table~\ref{tab:costbreakdown} separates one-time construction, feature
generation, and readout fitting. The proposed state collection is slower
than the orthogonal and CRJ implementations in this CPU code, despite its
six-member block structure, because it performs six small matrix-vector
updates and concatenations at every step. Its construction is inexpensive:
the six QR factorizations are on $50\times50$ matrices rather than one
$300\times300$ matrix. These constants explain why the asymptotic
$O(d_z^2/M)$ recurrent structure does not translate into an identical
factor-$M$ wall-clock improvement at $d_z=300$.

For a scalar regression target, every 300-state reservoir fits 300
coefficients and one intercept. On bounded NARMA-20, validation selects
$L=5$ for NG-RC, giving five linear features and
$5(5+1)/2=15$ quadratic features; its ridge readout therefore has only 21
trainable parameters including the intercept. Its poor BN20 result is not
caused by expensive optimization but by a validation-selected finite window
that does not represent the full recurrence.

For a one-layer GRU with input dimension $d_x$, hidden size $h$, and a
scalar output, the implementation has
\begin{equation}
N_{\rm GRU}=3hd_x+3h^2+6h+h+1
\label{eq:gru_params}
\end{equation}
parameters, where the $6h$ term accounts for the two bias vectors used by
each of the three gates. With $d_x=1$ and $h=32$, this gives 3393, matching
Table~\ref{tab:costbreakdown}. The timing comparison includes early-stopped
Adam optimization for the GRU but excludes the one-time hyperparameter
search for every model. It therefore measures the cost of fitting an
already selected model, not the total cost of model development.

\subsection{Effective Rank}

For centered training states with singular values $s_i$, we define
$p_i=s_i/\sum_j s_j$. The reported effective rank is
\begin{equation}
r_{\rm eff}=\exp\!\left(-\sum_i p_i\log p_i\right).
\label{eq:effective_rank}
\end{equation}
Unlike algebraic rank, $r_{\rm eff}$ decreases when most state energy is
concentrated in a few singular directions. It is consequently useful for
the ablation, but it is not a performance objective: the rotation-only
variant demonstrates that a larger $r_{\rm eff}$ can coexist with worse
prediction.

\bibliographystyle{IEEEtran}
\bibliography{refs}

@techreport{jaeger2001,
  author = {Jaeger, Herbert},
  title  = {The "Echo State" Approach to Analysing and Training Recurrent Neural Networks},
  number = {GMD Report 148},
  institution = {German National Research Center for Information Technology},
  year   = {2001}
}

@article{maass2002,
  author  = {Maass, Wolfgang and Natschl{\"a}ger, Thomas and Markram, Henry},
  title   = {Real-time Computing without Stable States: A New Framework for Neural Computation Based on Perturbations},
  journal = {Neural Computation},
  volume  = {14},
  number  = {11},
  pages   = {2531--2560},
  year    = {2002}
}

@article{tanaka2019,
  author  = {Tanaka, Gouhei and Yamane, Toshiyuki and H{\'e}roux, Jean Benoit and Nakane, Ryosho and Kanazawa, Naoki and Takeda, Seiji and Numata, Hidetoshi and Nakano, Daiju and Hirose, Akira},
  title   = {Recent Advances in Physical Reservoir Computing: A Review},
  journal = {Neural Networks},
  volume  = {115},
  pages   = {100--123},
  year    = {2019}
}

@article{manjunath2013,
  author  = {Manjunath, G. and Jaeger, Herbert},
  title   = {Echo State Property Linked to an Input: Exploring a Fundamental Characteristic of Recurrent Neural Networks},
  journal = {Neural Computation},
  volume  = {25},
  number  = {3},
  pages   = {671--696},
  year    = {2013}
}

@article{boyd1985,
  author  = {Boyd, Stephen and Chua, Leon O.},
  title   = {Fading Memory and the Problem of Approximating Nonlinear Operators with {Volterra} Series},
  journal = {IEEE Transactions on Circuits and Systems},
  volume  = {32},
  number  = {11},
  pages   = {1150--1161},
  year    = {1985}
}

@article{grigoryeva2018,
  author  = {Grigoryeva, Lyudmila and Ortega, Juan-Pablo},
  title   = {Echo State Networks are Universal},
  journal = {Neural Networks},
  volume  = {108},
  pages   = {495--508},
  year    = {2018}
}

@article{rodan2011,
  author  = {Rodan, Ali and Ti\v{n}o, Peter},
  title   = {Minimum Complexity Echo State Network},
  journal = {IEEE Transactions on Neural Networks},
  volume  = {22},
  number  = {1},
  pages   = {131--144},
  year    = {2011}
}

@article{rodan2012,
  author  = {Rodan, Ali and Tino, Peter},
  title   = {Simple Deterministically Constructed Cycle Reservoirs with Regular Jumps},
  journal = {Neural Computation},
  volume  = {24},
  number  = {7},
  pages   = {1822--1852},
  year    = {2012}
}

@inproceedings{saxe2014,
  author    = {Saxe, Andrew M. and McClelland, James L. and Ganguli, Surya},
  title     = {Exact Solutions to the Nonlinear Dynamics of Learning in Deep Linear Neural Networks},
  booktitle = {International Conference on Learning Representations (ICLR)},
  year      = {2014}
}

@inproceedings{arjovsky2016,
  title={Unitary evolution recurrent neural networks},
  author={Arjovsky, Martin and Shah, Amar and Bengio, Yoshua},
  booktitle={International conference on machine learning},
  pages={1120--1128},
  year={2016},
  organization={PMLR}
}

@article{gallicchio2017,
  author  = {Gallicchio, Claudio and Micheli, Alessio and Pedrelli, Luca},
  title   = {Deep Reservoir Computing: A Critical Experimental Analysis},
  journal = {Neurocomputing},
  volume  = {268},
  pages   = {87--99},
  year    = {2017}
}

@article{fujii2017,
  author  = {Fujii, Keisuke and Nakajima, Kohei},
  title   = {Harnessing Disordered-Ensemble Quantum Dynamics for Machine Learning},
  journal = {Physical Review Applied},
  volume  = {8},
  number  = {2},
  pages   = {024030},
  year    = {2017},
  doi     = {10.1103/PhysRevApplied.8.024030}
}

@article{nakajima2020,
  author  = {Nakajima, Keisuke},
  title   = {Physical Reservoir Computing---an Introductory Perspective},
  journal = {Japanese Journal of Applied Physics},
  volume  = {59},
  number  = {6},
  pages   = {060501},
  year    = {2020}
}

@article{mujal2021,
  author  = {Mujal, Pere and Mart{\'i}nez-Pe{\~n}a, Rodrigo and Nokkala, Johannes and Garc{\'i}a-Beni, Jorge and Giorgi, Gian Luca and Soriano, Miguel C. and Zambrini, Roberta},
  title   = {Opportunities in Quantum Reservoir Computing and Extreme Learning Machines},
  journal = {Advanced Quantum Technologies},
  volume  = {4},
  number  = {8},
  pages   = {2100027},
  year    = {2021}
}

@article{gks1976,
  author  = {Gorini, Vittorio and Kossakowski, Andrzej and Sudarshan, E. C. G.},
  title   = {Completely Positive Dynamical Semigroups of {$N$}-Level Systems},
  journal = {Journal of Mathematical Physics},
  volume  = {17},
  number  = {5},
  pages   = {821--825},
  year    = {1976}
}

@article{lindblad1976,
  author  = {Lindblad, G{\"o}ran},
  title   = {On the Generators of Quantum Dynamical Semigroups},
  journal = {Communications in Mathematical Physics},
  volume  = {48},
  number  = {2},
  pages   = {119--130},
  year    = {1976}
}

@article{suzuki2022,
  author  = {Suzuki, Yudai and Gao, Qi and Pradel, Karim and Yasuoka, Kenji and Yamamoto, Naoki},
  title   = {Natural Quantum Reservoir Computing for Temporal Information Processing},
  journal = {Scientific Reports},
  volume  = {12},
  pages   = {1353},
  year    = {2022}
}

@article{kobayashi2024_esp,
  author  = {Kobayashi, Shumpei and Tran, Quoc Hoan and Nakajima, Kohei},
  title   = {Extending Echo State Property for Quantum Reservoir Computing},
  journal = {Physical Review E},
  volume  = {110},
  number  = {2},
  pages   = {024207},
  year    = {2024}
}

@article{sannia2024,
  author  = {Sannia, Antonio and Mart{\'i}nez-Pe{\~n}a, Rodrigo and Soriano, Miguel C. and Giorgi, Gian Luca and Zambrini, Roberta},
  title   = {Dissipation as a Resource for Quantum Reservoir Computing},
  journal = {Quantum},
  volume  = {8},
  pages   = {1291},
  year    = {2024}
}

@article{buehner2006,
  author  = {Buehner, Michael and Young, Peter},
  title   = {A Tighter Bound for the Echo State Property},
  journal = {IEEE Transactions on Neural Networks},
  volume  = {17},
  number  = {3},
  pages   = {820--824},
  year    = {2006}
}

@article{yildiz2012,
  author  = {Yildiz, Izzet B. and Jaeger, Herbert and Kiebel, Stefan J.},
  title   = {Re-visiting the Echo State Property},
  journal = {Neural Networks},
  volume  = {35},
  pages   = {1--9},
  year    = {2012}
}

@book{petruccione2002,
  author    = {Breuer, Heinz-Peter and Petruccione, Francesco},
  title     = {The Theory of Open Quantum Systems},
  publisher = {Oxford University Press},
  year      = {2002}
}

@article{jaeger2007leaky,
  author    = {Jaeger, Herbert and Lukosevicius, Mantas and Popovici, Dan and Siewert, Udo},
  title     = {Optimization and Applications of Echo State Networks with Leaky-Integrator Neurons},
  journal   = {Neural Networks},
  volume    = {20},
  number    = {3},
  pages     = {335--352},
  year      = {2007}
}

@article{roy2007,
  author  = {Roy, Olivier and Vetterli, Martin},
  title   = {The Effective Rank: A Measure of Effective Dimensionality},
  journal = {European Signal Processing Conference (EUSIPCO)},
  pages   = {606--610},
  year    = {2007}
}

@article{atiya2000narma,
  author  = {Atiya, Amir F. and Parlos, Alexander G.},
  title   = {New Results on Recurrent Network Training: Unifying the Algorithms and Accelerating Convergence},
  journal = {IEEE Transactions on Neural Networks},
  volume  = {11},
  number  = {3},
  pages   = {697--709},
  year    = {2000}
}

@article{mackeyglass1977,
  author  = {Mackey, Michael C. and Glass, Leon},
  title   = {Oscillation and Chaos in Physiological Control Systems},
  journal = {Science},
  volume  = {197},
  number  = {4300},
  pages   = {287--289},
  year    = {1977}
}

@article{lorenz1963,
  author  = {Lorenz, Edward N.},
  title   = {Deterministic Nonperiodic Flow},
  journal = {Journal of the Atmospheric Sciences},
  volume  = {20},
  number  = {2},
  pages   = {130--141},
  year    = {1963}
}

@article{gauthier2021ngrc,
  author  = {Gauthier, Daniel J. and Bollt, Erik and Griffith, Aaron and Barbosa, Wendson A. S.},
  title   = {Next Generation Reservoir Computing},
  journal = {Nature Communications},
  volume  = {12},
  pages   = {5564},
  year    = {2021},
  doi     = {10.1038/s41467-021-25801-2}
}

@inproceedings{cho2014gru,
  author    = {Cho, Kyunghyun and van Merri{\"e}nboer, Bart and Gulcehre, Caglar and Bahdanau, Dzmitry and Bougares, Fethi and Schwenk, Holger and Bengio, Yoshua},
  title     = {Learning Phrase Representations using {RNN} Encoder--Decoder for Statistical Machine Translation},
  booktitle = {Proceedings of the Conference on Empirical Methods in Natural Language Processing (EMNLP)},
  pages     = {1724--1734},
  year      = {2014},
  doi       = {10.3115/v1/D14-1179}
}

@misc{vito2008airquality,
  author       = {De Vito, Saverio},
  title        = {Air Quality},
  year         = {2008},
  howpublished = {UCI Machine Learning Repository},
  doi          = {10.24432/C59K5F},
  url          = {https://doi.org/10.24432/C59K5F}
}

@article{sun2022systematic,
  title={A systematic review of echo state networks from design to application},
  author={Sun, Chenxi and Song, Moxian and Cai, Derun and Zhang, Baofeng and Hong, Shenda and Li, Hongyan},
  journal={IEEE Transactions on Artificial Intelligence},
  volume={5},
  number={1},
  pages={23--37},
  year={2022},
  publisher={IEEE}
}

@article{xu2026enhanced,
  title={An Enhanced Air Quality Forecasting Method Integrating Feature Mode Decomposition and Multi-layer Plasticity Echo State Network},
  author={Xu, Xinghan and Liu, Jianwei and Hu, Lei and Miao, Xingyi and Li, Yufeng and Liu, Siyu},
  journal={IEEE Transactions on Artificial Intelligence},
  year={2026},
  publisher={IEEE}
}

@article{strauss2012design,
  title={Design strategies for weight matrices of echo state networks},
  author={Strauss, Tobias and Wustlich, Welf and Labahn, Roger},
  journal={Neural computation},
  volume={24},
  number={12},
  pages={3246--3276},
  year={2012},
  publisher={MIT Press}
}

@article{mayer2017orthogonal,
  title={Orthogonal echo state networks and stochastic evaluations of likelihoods},
  author={Mayer, N Michael and Yu, Ying-Hao},
  journal={Cognitive Computation},
  volume={9},
  number={3},
  pages={379--390},
  year={2017},
  publisher={Springer}
}

\end{document}